\documentclass[10pt]{article}
\pdfoutput=1
\usepackage[utf8]{inputenc}
\usepackage[toc, page]{appendix}
\usepackage[font=small,labelfont=bf]{caption}
\usepackage{amsmath}
\usepackage{amssymb}
\usepackage{upgreek}
\usepackage{amsfonts} 
\usepackage{amsthm} 
\usepackage{color}
\usepackage{hyperref} 
\usepackage{setspace}
\usepackage{xcolor}
\usepackage{cancel}
\usepackage{enumerate}
\usepackage{Ari}
\usepackage{algorithm}          
\usepackage{algpseudocodex}     
\usepackage{nicematrix}
\usepackage[T1]{fontenc}
\usepackage{libertine} 

\usepackage{mdframed}

\hypersetup{colorlinks=true,linkcolor={red},citecolor={teal},urlcolor={teal}} 
\usepackage[total={6in,9in}]{geometry}
\usepackage{graphicx}
\graphicspath{{./images/}}

\newtheoremstyle{newstyle}      
{10pt} 
{10pt} 
{\itshape} 
{} 
{\bfseries} 
{.} 
{ } 
{} 

\theoremstyle{newstyle}
\newtheorem{theorem}{Theorem}[section]    
\newtheorem{definition}{Definition}[section]

\newtheorem{lemma}{Lemma}[section]

\renewcommand{\epsilon}{\varepsilon}

\let\oldnl\nl
\newcommand{\nonl}{\renewcommand{\nl}{\let\nl\oldnl}}

\title{Agnostic Membership Query Learning with Nontrivial Savings: New Results, Techniques}

\author{
Ari Karchmer
\\ Boston University \\ \href{mailto:arika@bu.edu}{arika@bu.edu} 
}
\begin{document}

\maketitle

\begin{abstract}%

Designing computationally efficient algorithms in the agnostic learning model (Haussler, 1992; Kearns et al., 1994) is notoriously difficult.
In this work, we consider agnostic learning with membership queries for touchstone classes at the \textit{frontier} of agnostic learning, with a focus on how much computation can be \textit{saved} over the trivial runtime of $2^n$. This approach is inspired by and continues the study of ``learning with nontrivial savings'' (Servedio and Tan, 2017). To this end, we establish multiple agnostic learning algorithms, highlighted by:

\begin{itemize}
  \item An agnostic learning algorithm for circuits consisting of a sublinear number of gates, which can each be any function computable by a sublogarithmic degree $k$ polynomial threshold function (the depth of the circuit is bounded only by size). This algorithm runs in time $2^{n -s(n)}$ for $s(n) \approx n/(k+1)$, and learns over the uniform distribution over unlabelled examples on $\{0,1\}^n$.
\item An agnostic learning algorithm for circuits consisting of a sublinear number of gates, where each can be any function computable by a $\sym^+$ circuit of subexponential size and sublogarithmic degree $k$. This algorithm runs in time $2^{n-s(n)}$ for $s(n) \approx n/(k+1)$, and learns over distributions of unlabelled examples that are products of $k+1$ \textit{arbitrary and unknown} distributions, each over $\{0,1\}^{n/(k+1)}$ (assume without loss of generality that $k+1$ divides $n$). 

\end{itemize}
Furthermore, we apply our new agnostic learning algorithms for these classes to also obtain algorithms for randomized compression, exact learning with membership and equivalence queries, and distribution-independent PAC-learning with membership queries.

Our core technique, which may be of independent interest, remixes the learning from natural proofs paradigm (Carmosino et al. 2016, 2017), so that we can tolerate concept classes fundamentally different than $\AC^0[p]$, and achieve \textit{fully} agnostic learning. We make use of communication-complexity-based natural proofs (Nisan, 1993), rather than natural proofs of Razborov (1987) and Smolensky (1987) for $\AC^0[p]$. 
\end{abstract}

\newpage 

\section{Introduction}

\textit{Agnostic} learning \cite{haussler1992decision, kearns1992toward} is an important generalization of PAC-learning \cite{valiant1984theory}. Agnostic learning is meant to more accurately capture a common approach to machine learning, where a predefined set of functions is explored in order to find the one achieving the least error on a set of data produced by some totally unknown process. Thus, roughly speaking, the objective of an \textit{agnostic learning algorithm} for a complexity class $\Lambda$ is to output a \textit{hypothesis} $h$ whose error in approximating an arbitrary concept is nearly as small as that of the \textit{best possible} hypothesis within $\Lambda$. The class $\Lambda$ is referred to as the \textit{touchstone class}.

Designing \textit{computationally efficient} (i.e., polynomial time) agnostic learning algorithms for expressive touchstone classes has historically been relatively hard. Even extremely simple touchstone classes such as parity functions are believed to be computationally hard to learn in the agnostic model \cite{blum1993cryptographic}.
Some positive results exist, however, including for piecewise functions \cite{kearns1992toward}, restricted fan-in two-layer neural nets \cite{lee1996efficient}, geometric patterns \cite{goldman1997agnostic}, decision trees, \cite{gopalan2008agnostically}, and halfspaces \cite{kalai2008agnostically}. 

If we take some combination of the common relaxations considered in computational learning theory, such as access to membership queries, distribution-specific learning, or super-polynomial runtime, more positive results become known. For instance, the famed polynomial time agnostic learning algorithm for parity functions due to \cite{goldreich1989hard} (also referred to sometimes as the KM algorithm after \cite{kushilevitz1991learning}), uses membership queries and requires a uniform distribution over unlabelled examples.
On the other hand, \cite{blum2003noise} and \cite{lyubashevsky2005parity} show agnostic learning for parities using only random examples, but restricting to the uniform distribution and allowing \textit{slightly} subexponential runtime.

In this work, we study agnostic learning under all three relaxations: access to membership queries, a varying degree of distribution-specificity, and super-polynomial runtime. Throughout the paper, we refer to agnostic learning algorithms that use membership queries as \textit{AMQ-learners}, and the learning model as \textit{AMQ-learning}. Distributional assumptions and/or super-polynomial runtimes are stated explicitly.

More specifically, the goal of this work is to increase the complexity of the touchstone classes that are known to be agnostically learnable. Thus, with respect to the runtime relaxation, we will focus merely on how much computation can be \textit{saved} over the \textit{trivial} runtime of $2^n$ (here $n$ is the number of binary ``inputs'' to the target concept). To summarize, we will try to design AMQ-learners that assume some knowledge about the distribution over unlabelled examples, and which run in time $2^{n-s(n)}$ for a \textit{savings function} $s(n)$. Clearly, when $s(n) = n-O(\log(n))$, runtime is polynomial. Generally speaking, we aim for $s(n) \in \omega(\log(n))$.
Throughout the paper, we refer to a learning algorithm in a specified model that runs in time $2^{n-s(n)}$ as $s(n)$-nontrivial, or simply nontrivial, if $s(n) \in \omega(\log(n))$.

This perspective is heavily inspired by the idea of \cite{servedio2017circuit}, who first explicitly considered learning with nontrivial savings. However, our learning model differs significantly, as \cite{servedio2017circuit} did not consider agnostic learning, but also did not make use of membership queries. Additionally, for technical reasons, they considered an online mistake-bound model of learning \cite{littlestone1988learning}, which is actually stronger than distribution-independent PAC-learning \cite{littlestone1988learning, blum1994separating} and equivalent to exact learning with equivalence queries only \cite{angluin1988queries}. In the online mistake-bound model, they obtained nontrivial learning algorithms for $\AC^0$ circuits with a few LTF-gates, or augmented with $\mod p$-gates, in addition to full-basis formulas, branching programs, and span programs of some fixed polynomial size (all less than $n^2$).

Understanding how much computation can be saved is a relatively new goal in computational learning theory, which, on the other hand, has a longer, fruitful history in complexity theory. For instance, there exist many examples of non-trivial upper bounds slightly better than $2^n$ for counting or satisfiability algorithms for CNFs, and other $\NP$-hard or $\#\P$-hard problems (see \cite{paturi1997satisfiability, schoning1999probabilistic, paturi2005improved, schuler2005algorithm, fomin2013exact, impagliazzo2012satisfiability}). Thus, as \cite{servedio2017circuit} pointed out, finding out the extent of what can be learned with nontrivial savings is worthwhile to pursue in order to push computational learning theory forward. Additionally, as we will discuss later, a direct implication from nontrivial learning algorithms to \textit{compression} algorithms for Boolean circuit classes provides a concrete application for our study.

\subsection{Our Results}\label{intro:ourResults}


We present a variety of new nontrivial AMQ-learning algorithms. 
For the sake of clarity, we now present somewhat weaker and slightly informal statements that still convey the main results of this work. The formal and more technically precise statements proved in the body of the paper are easily seen to imply them. We compare these algorithms to related ones in Section \ref{related}. An extended technical overview follows in Section \ref{sec:ExtendedTechnical}, which includes a precise definition of AMQ-learning.

\paragraph{AMQ-learning over the uniform distribution.} First, we obtain nontrivial AMQ-learners, with respect to a uniform distribution over unlabelled examples on $\{0,1\}^n$, for functions of degree $k$ polynomial threshold functions (PTFs). 
Specifically, the following touchstone classes:

\begin{itemize}
    \item \textbf{Class 1:} Circuits of at most $n^{0.99}$ gates, where each gate computes any function computable by a PTF of degree $k \le \log(n)^{0.99}$, in time $2^{n -s(n)}$ for $s(n) \approx n/(k+1)$.
    \item \textbf{Class 2:} Decision trees of at most $n^{0.99}$ depth, where each node is allowed to make a query computed by a PTF of degree $k \le \log(n)^{0.99}$, in time $2^{n -s(n)}$ for $s(n) \approx n/(k+1)$.
\end{itemize}
We use $\approx$ to suppress additive factors that are sublinear in $n$, and an unavoidable logarithmic dependency on the reciprocal of the error rate of the optimal hypothesis in the touchstone class. 

\paragraph{AMQ-learning over $t$-product distributions.} We also obtain nontrivial AMQ-learners when the distribution over unlabelled examples on $\{0,1\}^n$ is a $t$\textit{-product}. A $t$-product is a distribution that is defined by $t$ arbitrary distributions, each over $\{0,1\}^{n/t}$ (assume without loss of generality that $t$ divides $n$). The $t$-product distribution is sampled by taking independent samples from each of the $t$ distributions, and concatenating them to form an element of $\{0,1\}^n$. The ability to handle $t$-product distributions over unlabelled examples is a significant upgrade over the uniform distribution, for at least two reasons. First, this class of distributions allows for intricate dependencies between some of the bits of the unlabelled examples. Second, the $t$ distributions need not even be known to the algorithm. That is, the algorithm is not specific to the choice of the $t$ distributions that make up the $t$-product.\footnote{If we could upgrade further to arbitrary distributions, then we could remove the need for membership queries (see \cite{feldman2009power}). This is an interesting goal for future research.}

We give AMQ-learners for functions of $\sym^+$ circuits of subexponential size and degree $k$ (see touchstone classes 3 and 4), where the distribution over unlabelled examples is an unknown $(k+1)$-product.
A $\sym^+$ circuit of size $s$ and degree $k$ is defined by a pair $(p, \theta)$, where $p$ is an $n$-variable degree-$k$ multilinear polynomial over the integers. Size $s$ means that the magnitude of the coefficients of $p$ is at most $s$. On the other hand, $\theta: \mathbb{Z} \rightarrow \{-1,1\}$ is an arbitrary function. The $\sym^+$ circuit $(p, \theta)$ evaluates by computing the function $s: \{0,1\}^n \rightarrow \{-1,1\}$ defined as $s(x) = \theta(p(x))$.

\begin{itemize}
    \item \textbf{Class 3:} Circuits of at most $n^{1-\epsilon}$ gates, where each gate computes any function computable by a $\sym^+$ circuit of size $2^{n^{\zeta}}$ ($\zeta < \epsilon$) and degree $k \le \log(n)^{0.99}$, in time $2^{n -s(n)}$ for $s(n) \approx n/(k+1)$.
    \item \textbf{Class 4:} Decision trees of at most $n^{1-\epsilon}$ depth, where each node is allowed to make a query computed by a size $2^{n^{\zeta}}$ ($\zeta < \epsilon$) and degree $k \le \log(n)^{0.99}$ $\sym^+$ circuit, in time $2^{n -s(n)}$ for $s(n) \approx n/(k+1)$.
\end{itemize}

Note, degree $k$ PTFs with sum of coefficients less than $s$ are a subclass of size $s$, degree $k$ $\sym^+$ circuits. So, to expand the class of distributions over unlabelled examples handled by the AMQ-learner for touchstone classes 1 and 2, we have to modify the touchstone class by bounding the weights of each PTF-gate slightly. However, size $s$, degree $k$ $\sym^+$ circuits are still more general than degree $k$ PTFs with total weight bounded by $s$. Because of this, touchstone classes 3 and 4 are not necessarily supersets of touchstone classes 1 and 2.

\subsection{Some Applications of Our Results}\label{intro:apps}

Several applications of the nontrivial AMQ-learners for classes 1-4 in are in order. We refer to Section \ref{sec:compression} for precise statements and further discussion of these applications. 

\paragraph{Compression Algorithms for Classes 1-4.}
Randomized compression algorithms for Boolean functions \cite{chen2015mining} are obtained in a simple and generic way from membership query learning over the uniform distribution (this was noticed already in \cite{carmosino2016learning, servedio2017circuit}). Even though our results are AMQ-learners, they can be used in the \textit{realizable} setting as membership query learning algorithms as needed. Therefore, we also obtain randomized compression algorithms for touchstone classes 1-4, which were not known before. 
Our compression results also provide further support for a hypothesis of \cite{chen2015mining}, which is closely related to the questions of \cite{servedio2017circuit}. Their hypothesis is that every currently known natural circuit lower bound for circuits of type $\Lambda$ yields a compression algorithm for $\Lambda$. 

\paragraph{Nontrivial Exact Learning and Distribution-Independent PAC-Learning for Classes 1-4.} 
In a similar fashion, randomized learning algorithms in the exact learning with membership and equivalence queries model \cite{angluin1988queries} are obtained from membership query learning over the uniform distribution. Additionally, it is also known (by standard arguments -- see Section 2.4 of \cite{angluin1988queries}) that algorithms for exact learning with membership and equivalence queries imply distribution-independent membership query learning algorithms in the realizable PAC model. Hence, we also obtain nontrivial exact learning and distribution-independent PAC learning algorithms for classes 1-4, which were not known prior.

\subsection{Our Approach} 

The approach of \cite{servedio2017circuit} was to convert circuit lower bound methods (random restrictions, Nečiporuk's method) into nontrivial learning algorithms in the online mistake bound model. Because of this, they asked (see section 5 of \cite{servedio2017circuit} for exact quotations):
\begin{enumerate}
    \item Can other proof techniques from computational complexity be used to obtain other nontrivial learning algorithms?
    \item Many circuit classes have known lower bounds, but not nontrivial learning algorithms (in any model); can we design nontrivial learning algorithms for such classes?
\end{enumerate}

The approach we take is led by this line of questioning. Specifically, we translate circuit lower bounds proved by a communication complexity based method due to \cite{nisan1993communication} into nontrivial AMQ-learners. This goes to answer question 1 above. It also means that we obtain answers to question 2, because, fixing a degree $k$, none of touchstone classes 1-4 can compute the generalized inner product function of degree $k+1$ \cite{nisan1993communication}, but prior to this work no nontrivial learning algorithms were known for these classes (in any learning model). The essential details and definitions of communication complexity in the relevant communication models will be explained in sufficient depth in the technical overview (Section \ref{sec:ExtendedTechnical}); we point to \cite{kushilevitz1996communication} for further reference.

At a very high level, the method of \cite{nisan1993communication} is the following. First, identify a function $\phi: \{0,1\}^n \rightarrow \{0,1\}$ that requires \textit{high} communication complexity (e.g., $\Omega(n)$). This can be in any communication model, such as $k$-party number-on-forehead (NOF), two-party randomized, or two-party deterministic. Then, identify a circuit class $\cC$ (with an associated size parameter $s(n)$) such that, for every function $g$ computable by a circuit of type $\cC$ and size $s(n)$, $g$ is computable by a \textit{low} communication protocol in that communication model. \textit{Low} communication cost indicates a relatively small function of $s(n)$ like $\log(s(n))$. Finally, one can conclude that $f$ requires large circuits of type $\cC$.  


At the core of Nisan's method is the \textit{upper bound} on communication complexity for the circuit class $\cC$. 
Thus, by specifically showing that communication complexity \textit{upper bounds} in the $k$-party NOF model imply, in a general sense, approximately $n/k$-nontrivial AMQ-learners, we end up translating Nisan's general \textit{lower bound} method into nontrivial AMQ-learners. 
Since circuit lower bounds for explicit functions can be proved for each of touchstone classes 1-4 using Nisan's method, this suffices to derive our AMQ-learners.
In Section \ref{sec:ExtendedTechnical}, we give an in-depth explanation of the construction of the AMQ-learner for deterministic NOF protocols.

\begin{figure}[H]
\centering
\includegraphics[scale=0.15]{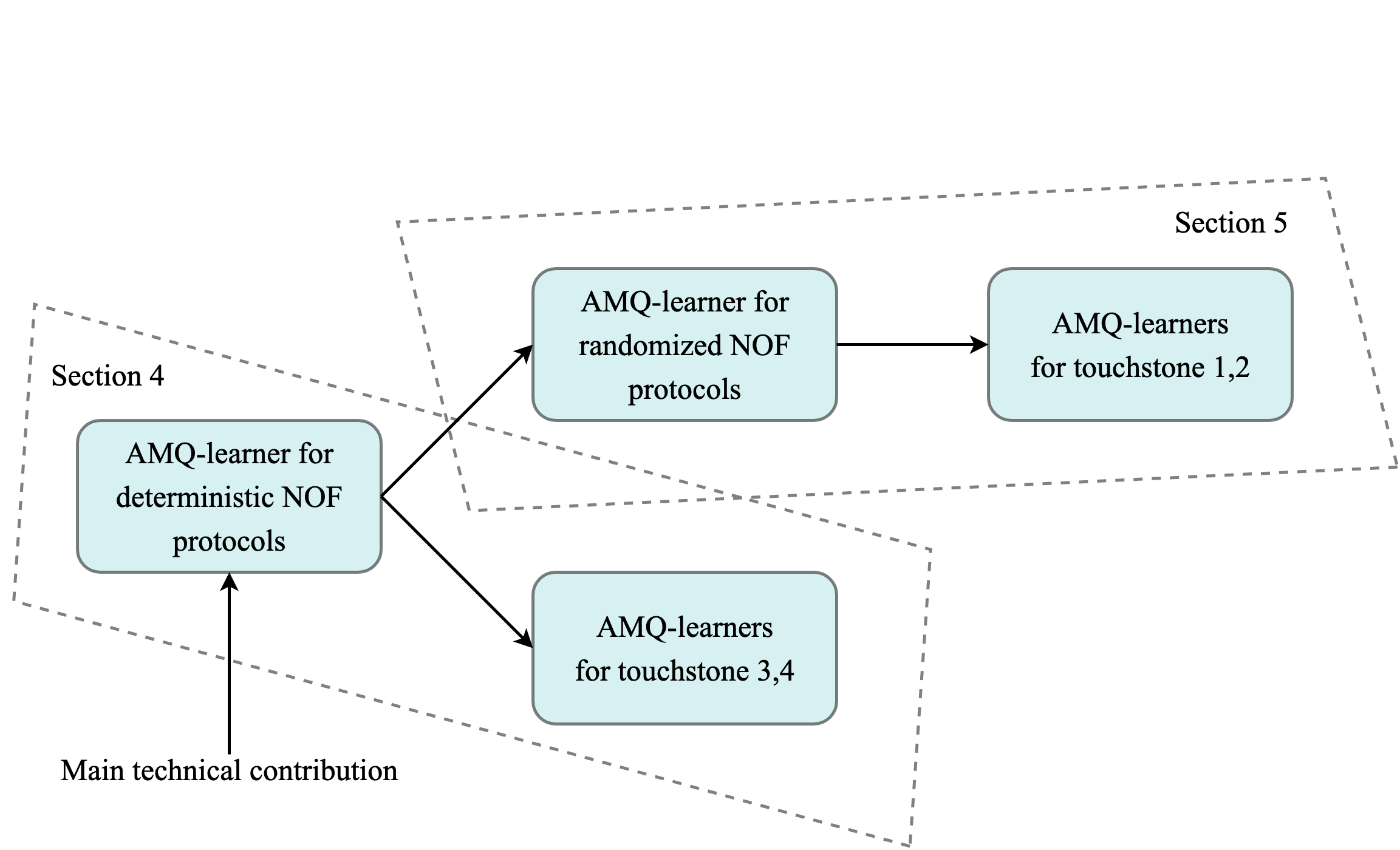}
\caption{The progression of our constructions.}
\end{figure}

\subsection{Related Results and Discussion}\label{related}

Not many touchstone classes that resemble ours are known to be nontrivially learnable, whether in AMQ-learning or the exact learning model. The most related results are the learning algorithms of \cite{servedio2017circuit} and \cite{carmosino2017agnostic}. We now compare our results to each. 

For comparison with \cite{servedio2017circuit}, let us focus on our results in the nontrivial exact learning model (see Section \ref{intro:apps}), which is the focus of their work.
Classes 1 and 2 are possibly more interesting than the classes shown to be nontrivially learnable in the exact learning with (only) equivalence queries by \cite{servedio2017circuit}. For example, LTF-of-$\AC^0$ and Parity-of-$\AC^0$ -- these circuit classes cannot compute the parity function or an LTF, respectively, while touchstone classes 1 and 2 can compute both. Additionally, our AMQ-learners for classes 1 and 2 are $n/O(1)$-nontrivial (when fixing a constant degree $k$ for the PTFs), while \cite{servedio2017circuit} only obtain $n/\polylog(n)$-nontrivial learners for LTF-of-$\AC^0$ and Parity-of-$\AC^0$.
However, it is important to note that our results are not necessarily formally stronger than \cite{servedio2017circuit} because we require membership queries, while \cite{servedio2017circuit} do not.

For comparison to \cite{carmosino2017agnostic}, let us focus on our results in the AMQ-learning model, which is their focus as well. \cite{carmosino2017agnostic} obtain a uniform-distribution-specific algorithm with some relatively mild agnostic guarantees for $\AC^0[p]$ circuits, for any prime $p$. More specifically, the AMQ learner for $\AC^0[p]$ finds a hypothesis with error as much as $\polylog(n)$ times the optimal hypothesis in $\AC^0[p]$.
In contrast, our AMQ-learner for classes 3 and 4 is \textit{fully} agnostic (i.e., it does not have the $\polylog(n)$ loss factor). Additionally, our AMQ-learner works over $k$-product distributions, rather than just the uniform distribution. In these two respects, our algorithm compares favorably. Plus, touchstone classes 3 and 4 are not contained in $\AC^0[p]$ for any prime $p$. However, their algorithm is much faster -- quasi-polynomial time -- while ours is approximately $n/2$-nontrivial. 

\paragraph{Could we improve our results?} An interesting note regarding touchstone classes 3 and 4 is that seemingly small improvements would be a breakthrough in computational learning theory (indeed classes 3 and 4 reach the AMQ-learning \textit{frontier}). To see this, recall that touchstone classes 3 and 4 are classes of functions of subexponential size but sublogarithmic degree $\sym^+$ circuits. Now observe that a \textit{quasipolynomial} size $\sym^+$ circuit with \textit{polylogarithmic} degree can compute all of $\ACC^0$ \cite{beigel1994acc}. Nontrivial AMQ-learning of $\ACC^0$ is an open problem, even when the distribution over unlabelled examples is fixed to be uniform, and the learning algorithm is not even agnostic. Thus, extending our AMQ-learner to handle a circuit with even \textit{one single} gate that computes a quasi-polynomial size, polylogarithmic degree $\sym^+$ function would be a breakthrough in learning theory. In fact, since nontrivial learning algorithms for $\ACC^0$ can be ``sped-up'' (by a result of \cite{oliveira2016conspiracies}), we would obtain a learning algorithm in time $2^{n^{\epsilon}}$, for any $\epsilon>0$. On the other hand, pseudorandom functions with exponential security computable by $\ACC^0$ circuits are conjectured \cite{boneh2018exploring}. These conjectures preclude $n/O(1)$-nontrivial AMQ-learners for $\ACC^0$.

\section{Extended Technical Overview}\label{sec:ExtendedTechnical}

In this section, we provide an in-depth overview of the technique behind our main technical contribution, from which all the AMQ-learners for touchstone classes 1-4 are derived. The main technical contribution constructs the AMQ-learner over $k$-product distributions for the touchstone class of functions admitting $k$-party NOF communication protocols of cost $c$ (we will set $c$ as a function of $n$ and $k$ later; for now, think of $k$ as a constant and $c:= \Theta(n^{0.99})$). 

We begin by defining $k$-party NOF communication, then define nontrivial AMQ-learning precisely, and finally walk through the main technical contribution.

\paragraph{$k$-party NOF communication.} The $k$-party NOF communication model is the following. There are $k$ parties, each having unbounded computational power, who try to collectively compute a function. The input to the function is separated into $k$ parts, each containing $n/k$ of the inputs (assume without loss of generality that $n$ is a multiple of $k$), and the $i^{th}$ party sees all parts except the $i^{th}$. The communication between the parties is by broadcast: any party can send a bit to all others simultaneously.

All parties may transmit messages according to a fixed protocol. The protocol determines, for every sequence of bits transmitted up to that point (the transcript), whether the protocol is finished (as a function of the transcript), or if (and which) party writes next (as a function of the transcript) and what that party transmits
(as a function of the transcript and the input of that party). Finally, the last bit transmitted is the output of the protocol, which is a value in $\{-1, 1\}$. The complexity measure of the protocol is the total number of bits transmitted by the parties.

\begin{definition}[$\dCC{k}{c}$ class]
$\dCC{k}{c}$ is defined to be the class of functions $f: \{0,1\}^n \rightarrow \{-1,1\}$ that, for any partition of the $n$ variables into $k$ length $n/k$ parts, can be computed by a $k$-party communication protocol with complexity $c$.
\end{definition}
Again, for now, one can think of $n$ as being a multiple of $k$ without loss of generality.
Within this extended technical overview, we will use the notation $[x_1, \cdots  x_k]$ to denote the concatenation of the $k$ parts $x_1, \cdots  x_k$, according to the appropriate partition, which is implicit. For simplicity, it is appropriate to think of the partition as contiguous blocks of $n/k$ bits, from ``left to right'' (depicted later in Figure \ref{fig:design-example2}). 


\paragraph{AMQ-learning.}

Valiant's PAC-learning model operates in the \textit{realizable} setting, where the concept $f: \{0,1\}^n \rightarrow \{-1,1\}$ is assumed to be part of a fixed class. In agnostic learning, the concept is modelled instead as an arbitrary \textit{distribution} $\cD$ over $\{0,1\}^n \times \{-1,1\}$. For $(x,y) \sim \cD$, $x$ is called the \textit{example}, and $y$ is called the \textit{label}. Note that the label for one example may be randomized. We refer to the marginal distribution $\rho$ over $x$ for $(x,y) \sim \cD$ as the distribution over unlabelled examples, or example distribution for short.

The goal of agnostic learning for a touchstone class $\Lambda$ is to find a hypothesis $h: \{0,1\}^n \rightarrow \{-1,1\}$ such that, with probability $1-\delta$, $h$ satisfies
\[
    \Ex{(x,y) \sim \cD}{h(x) \cdot y} \ge \opt{\Lambda} - \epsilon
\]
where the learning algorithm is given $\epsilon, \delta$ as input, and $\opt{\Lambda}$ is defined as
\[
   \opt{\Lambda} := \max_{c \in \Lambda} \Ex{(x,y)\sim \cD}{c(x) \cdot y}
\]
Throughout this paper, we will make assumptions about $\rho$, the distribution over unlabelled examples. We specify that we obtain agnostic learning when $\rho$ is restricted to be part of some class of distributions (e.g., $k$-product distributions). In AMQ-learning, we allow the learning algorithm to have access to a membership query oracle, which works as follows. The AMQ-learner can submit an example $z$, of its choice, and receive back the label $y$ with probability $\Pr[(x,y) \sim \cD | x = z]$.

In nontrivial AMQ-learning, the two notions of complexity that we consider are:

\begin{itemize}
    \item Query complexity: The total number of membership queries made by the AMQ-learner, in the worst-case.
    \item Time complexity: The worst-case runtime complexity of the AMQ-learner.
\end{itemize}

Normally, query and time complexity are measured as functions of $n, \epsilon, \delta$. This can make statements about complexity a bit messy in the nontrivial learning setting; we will fix $\delta := 2/3$, and $\epsilon := 2^{-n^{0.99}}$, and thus bound complexity purely as a function of $n$. Recall, an AMQ-learner is $s(n)$-nontrivial if time complexity is at most $2^{n-s(n)}$.

\paragraph{Towards the AMQ-learner.} 


Towards the full AMQ-learner, we construct a \textit{weak} AMQ-learner. A weak AMQ-learner essentially preserves a small portion of the correlation between the optimal hypothesis in the touchstone, and the target concept. We use a standard notion due to \cite{feldman2009distribution, kanade2009potential}.

\begin{definition}[Weak agnostic learning]
    For $0 < \alpha \le \gamma < 1$, a $(\gamma, \alpha)$-weak agnostic learner for $\cC$ is an algorithm $\cA$ that (when receiving some oracle access to a target concept $\cD$) outputs a hypothesis $h: \{0,1\}^n \rightarrow \{-1,1\}$ such that 
    \[
        \opt{\cC} = \max_{c \in \cal{C}} \Ex{(x,y)\sim \cD}{c(x)y} \ge \gamma \implies \Pr_{\cA}\left[\Ex{(x,y)\sim \cD}{h(x)y} \ge \alpha\right] \ge 2/3
    \]
\end{definition}

We construct a weak AMQ-learner of this kind, because \cite{feldman2009distribution, kanade2009potential} show that it can be \textit{boosted}, using \textit{distribution-specific} boosting algorithms, into a full-blown agnostic learning algorithm. This means that the boosting algorithm only ever invokes the weak AMQ-learner on a single distribution over unlabelled examples. This is important for us, because our initial result does not handle arbitrary distributions over unlabelled examples.

For the sake of simplicity in exposition, the main technical contribution is stated below with respect to the uniform distribution over unlabelled examples, and the concept is considered an arbitrary \textit{function}, rather than distribution. The corresponding theorem proved in the body of the paper extends this to $k$-product distributions over unlabelled examples (as discussed previously), as well as concepts that are arbitrary distributions. 

Define $\corr{\Lambda} := \max_{g \in \Lambda} \Ex{z \sim \{0,1\}^n}{g(z) \cdot f(z)}$.

    
\begin{theorem}[Main technical contribution]\label{intro:mainTech_weak}
Let $f :\{0,1\}^n \rightarrow \{-1,1\}$. There exists an oracle algorithm $\cA$ with the guarantee that if $\corr{\dCC{k}{c}} \ge \gamma$, then $\cA^{f}$ outputs $h: \{0,1\}^n \rightarrow \{-1,1\}$ such that 
\[
\Pr_{\cA}\left[\Pr_{z \sim \{0,1\}^n}\left[{h(z) = f(z)}\right] \ge 1/2+\alpha\right] \ge 2/3
\]
with:
\begin{itemize}
    \item $\alpha := \Omega(\gamma \cdot 2^{-c{2^k}-k})$,
    \item query complexity $q := 2^{n-s(n)}$, for $s(n) = n/k-\log(k) +4(\log(\gamma) - c2^k +k))$
    \item time complexity $t := O(q)$.
\end{itemize}
Here, $z \sim \{0,1\}^n$ denotes sampling over the uniform distribution.
\end{theorem}

The main tool we use towards this theorem is the ``$k$-party norm''  of a function $f: \{0,1\}^n \rightarrow \{-1,1\}$, denoted $R_k(f)$, defined as:

\begin{definition}[$k$-party norm]
For $f: \{0,1\}^n \rightarrow \{-1,1\}$, the $k$-party norm of $f$ is defined as
\begin{equation}\label{intro:2partynorm}
    R_k(f) := \Ex{x^0_1, \cdots x^0_k, x^1_1, \cdots , x^1_k \sim \{0,1\}^{n/k}}{\prod_{\epsilon_1,\epsilon_2 \in \{0,1\}} f([x^{\epsilon_1}_1, \cdots  x^{\epsilon_2}_k])}
\end{equation}

\end{definition}
Again, we assume without loss of generality that $k$ divides $n$.

The main insight towards a weak AMQ-learner for $\dCC{k}{c}$ is that $R_k(f)$ upper bounds the correlation of $f$ with functions computable by deterministic $k$-party NOF communication protocols. 
Proved in all three of \cite{chung1993communication, raz2000bns, viola2007norms}, is the following bound: 
\begin{theorem}[$k$-party correlation bound]
\label{intro:thm:corrBound}
For every function $f : \{0,1\}^n \rightarrow \{-1,1\}$, 
\begin{equation}\label{intro:corrBound}
    \corr{\dCC{k}{c}} = \max\limits_{\pi \in \dCC{k}{c}}\left|\Ex{x}{f(x) \cdot \pi(x)}\right| \le 2^c \cdot R_k(f)^{1/2^k}
\end{equation}
for $x$ uniformly distributed over $\{0,1\}^n$.
\end{theorem}

This correlation bound can be used to obtain a \textit{natural proof}, in the sense of \cite{razborov1997natural}, against $\dCC{k}{c}$. This has observed for many years as folklore, but explicitly shown in \cite{karchmer2023distributional}. On first thought, it seems we should then be able to obtain learning algorithms using the technique of \cite{carmosino2016learning, carmosino2017agnostic}, but this does not work, because the ``combinatorial designs'' within the Nisan-Wigderson PRG, used at the heart of their construction, require too much communication to compute. This trivializes the bound in (\ref{intro:corrBound}).

The way we get around this is as follows: we use the combinatorial design that \textit{arises naturally} from the $k$-party norm itself, and then use the $k$-party correlation bounds as a \textit{distinguisher}. Observe: $R_k(f)$ is the expectation of the product of $f$ computed on all combinations of a list of $2k$, $n/k$-bit ``slices.'' So, we view the $2n$ bits in this list as a ``seed,'' and each of the $2^k$ combinations as the designs over slices of this seed. See Figure \ref{fig:design-example2} for a graphical depiction.

\begin{figure}[H]
\centering
\includegraphics[scale=0.12]{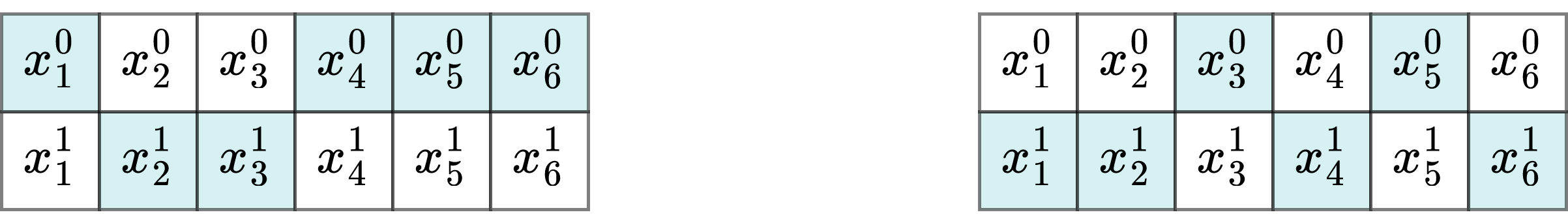}
\caption{Two examples of designs naturally arising from $R_6(f)$. Each $x^i_j$ is a $n/6$ bit slice of the $2n$ bit seed. Design indices `011000' and `110101' are highlighted on the left and right images respectively.}
\label{fig:design-example2}
\end{figure}

In other words, if we generate a random seed $x$, and then query the concept $f$ at every point $z_i$ indicated by the seed projected to design $i \in \{0,1\}^k$ (see Figure \ref{fig:projected-design}), then the distribution over the labels $\langle f(z_i)\rangle_{i \in \{0,1\}^k}$ can be distinguished from a uniformly random string in $\{-1,1\}^{2^k}$! The distinguisher simply takes the product of the bits. By design, this product has expectation at least $(\gamma 2^{-c})^{2^k}$, which follows directly from Theorem \ref{intro:thm:corrBound} (if we assume that the concept $f$ has $\gamma$ correlation with $\dCC{k}{c}$). For a random string, the expectation is zero. Therefore the distinguishing advantage is $(\gamma 2^{-c})^{2^k}$.

\begin{figure}[H]
\centering
\includegraphics[scale=0.17]{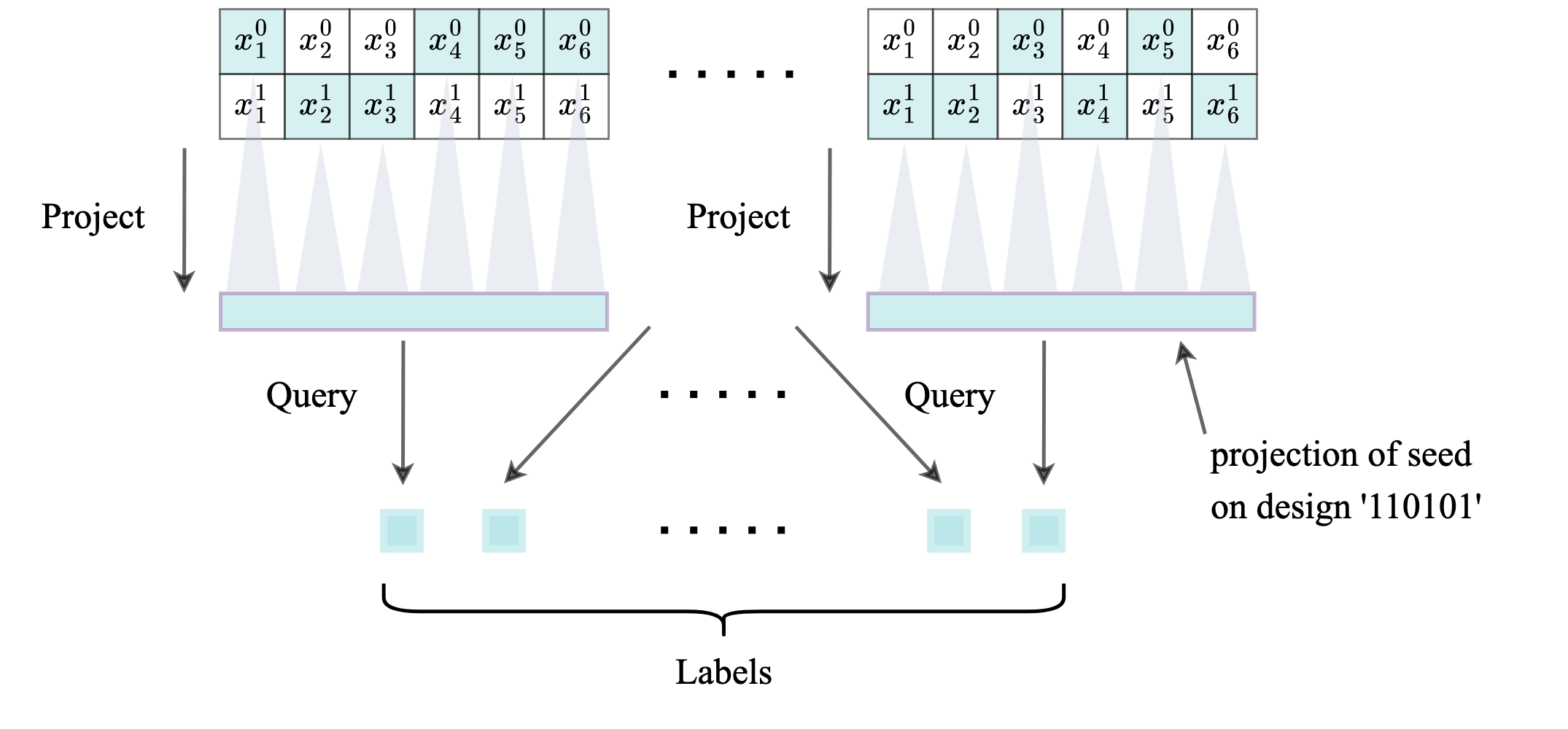}
\caption{Example of projected designs to form queries. The claim is that the sequence of $2^k$ labels can be distinguished from a uniformly random string in the set $\{-1,1\}^{2^k}$, whenever the concept has a large $k$-party norm.}
\label{fig:projected-design}
\end{figure}

Distinguishing labels queried in this randomized manner from a random string is sufficient to obtain a randomized prediction algorithm for $f$ with weak accuracy approximately $1/2+(\gamma 2^{-c})^{2^k}$, using a (rather involved) hybrid argument.
Then, we can proceed as \cite{carmosino2016learning} does: we use a randomized pre-processing stage to ``derandomize'' the randomized predictor. That is, we convert the randomized predictor to a randomized process that generates deterministic hypotheses with a weaker (but similar) prediction accuracy, with some non-negligible probability.
Finally, we use a constructive averaging argument to obtain a deterministic hypothesis with weak accuracy, with very high probability.



\paragraph{A brief comparison to \cite{carmosino2016learning}.} 

Interestingly, the design of our AMQ-learner follows a similar path to the influential learning algorithm for $\AC^0[p]$ from \cite{carmosino2016learning, carmosino2017agnostic} (the ``learning from natural proofs paradigm''), but with crucial twists along the way. 
In fact, our algorithm can be viewed as a remix of \cite{carmosino2016learning}, where the combinatorial designs and natural proof are different. The combinatorial designs we use follow the shape outlined in Figure \ref{fig:design-example2}, while \cite{carmosino2016learning} use the Nisan-Wigderson designs. The natural proof we leverage is useful against $\dCC{k}{c}$ and due to the sequence of works \cite{nisan1993communication, chung1993communication, raz2000bns, viola2007norms, karchmer2023distributional}, while \cite{carmosino2016learning} leverage natural proofs for $\AC^0[p]$ due to \cite{razborov1987lower, smolensky1987algebraic}.

Additionally, \cite{carmosino2017agnostic} extended the technique of \cite{carmosino2016learning} to AMQ-learning, over the uniform distribution, assuming \textit{tolerant} natural proofs.\footnote{Essentially, a tolerant natural proof can be thought of as a natural proof of an average-case lower bound.} However, their techniques still did not produce any strongly agnostic learning algorithm for $\AC^0[p]$, since the natural proofs of \cite{razborov1987lower, smolensky1987algebraic} do not have the requisite level of tolerance. In our case, Theorem \ref{intro:thm:corrBound} implies a highly tolerant natural proof, as it implies a strong \textit{average-case} lower bound against $\dCC{k}{c}$. Therefore our method obtains a full AMQ-learner, whereas \cite{carmosino2017agnostic} got a (much) weaker form of AMQ learning, which only obtains a hypothesis with error approximately as small as $\polylog(n)$ times the error rate of the optimal hypothesis within $\AC^0[p]$.





\section{Preliminaries}\label{prelims}

We cover basics of circuits. Other important definitions used in the paper are defined as needed throughout.

\paragraph{Complexity classes, circuits, gates, etc.} 

We consider various circuit classes with different bases (all being defined previously in the literature). $\AC^0$ is the class of constant depth, polynomial size, unbounded fan-in ${\sf AND/OR/NOT}$ circuits. $\AC^0[p]$ is the class of constant depth, polynomial size, unbounded fan-in ${\sf AND/OR/NOT/MOD}p$ circuits, where $p \in \nat$ is a prime number. An ${\sf XOR}$ gate takes the sum modulo 2 of its inputs.

An linear threshold function (LTF) gate computes an LTF defined by $t(x_1, \cdots x_m) := \sum^m_{i=1} w_ix_i \ge^? \theta$, which outputs 1 if and only if the sum of the inputs weighted by real coefficients $w_1, \cdots w_m$ exceeds a threshold $\theta$. When the weights are fixed to be $1$ and $\theta = m/2$, we call it a majority gate. A polynomial threshold function (PTF) is defined by a polynomial $p(x_1, \cdots, x_n)$ with real number coefficients. The $p(x)$ for input $x \in \{0,1\}^n$ is 1 if $p(x_1, \cdots, x_n) \ge 0$ and -1 otherwise. Note, the domain of the PTF is $\{0, 1\}^n$, so different polynomials can define identical PTFs. The degree of a
PTF is the degree of the polynomial $p$.
We use ${\sf PT}{\text -}{\sf Ckt}[k,m]$ to denote the class of circuits consisting of at most $m$ \textit{gates}, which can each compute a PTF of degree $k$, and ${\sf PT}{\text -}{\sf Dt}[k,d]$ to denote the class of decision trees consisting of at most depth $d$, with internal nodes computing PTFs of degree $k$.

A $\sym^+$ circuit of size $s$ and degree $d$ is a pair $(p, \theta)$, where $p$ is an $n$-variable degree-$d$ multi-linear polynomial over the integers. Size $s$ means that the magnitude of the coefficients of $p$ is at most $s$. On the other hand, $\theta: \mathbb{Z} \rightarrow \{-1,1\}$ is any function. The $\sym^+$ circuit $(p, \theta)$ evaluates by computing the function $s: \{0,1\}^n \rightarrow \{-1,1\}$ defined as $s(x) = \theta(p(x))$. We also consider circuits, where each gate computes any function computable by a $\sym^+$ circuit.
We denote by $\sym^+{\text -}{\sf Ckt}[k, t, m]$ the class of circuits consisting of at most $m$ gates, which can each compute any function computable by a $\sym^+$ circuit of degree $k$ and size $t$. We denote by $\sym^+{\text -}{\sf Dt}[k, t, d]$ the class of decision trees of depth $d$, with node queries computable by any $\sym^+$ circuit of degree $k$ and size $t$.


\section{Learning $k$-Party NOF Communication Protocols}\label{section:MQ-learning}

In this section, we will construct an agnostic learning algorithm for any touchstone class that has an efficient $k$-party deterministic number-on-forehead (NOF) communication protocol (defined previously but repeated below for the reader's convenience).

The $k$-party NOF communication model is the following. There are $k$ parties, each having unbounded computational power, who try to collectively compute a function. The input to the function is separated into $k$ parts, each containing $n/k$ of the inputs (assume without loss of generality that $n$ is a multiple of $k$), and the $i^{th}$ party sees all parts except the $i^{th}$. The communication between the parties is by broadcast: any party can send a bit to all others simultaneously.

All parties may transmit messages according to a fixed protocol. The protocol determines, for every sequence of bits transmitted up to that point (the transcript), whether the protocol is finished (as a function of the transcript), or if (and which) party writes next (as a function of the transcript) and what that party transmits
(as a function of the transcript and the input of that party). Finally, the last bit transmitted is the output of the protocol, which is a value in $\{-1, 1\}$. The complexity measure of the protocol is the total number of bits transmitted by the parties.

\begin{definition}[$\dCC{k}{c}$ class]
$\dCC{k}{c}$ is defined to be the class of functions $f: \{0,1\}^n \rightarrow \{-1,1\}$ that, for any partition of the $n$ variables into $k$ length $n/k$ parts, can be computed by a $k$-party communication protocol with complexity $c$.
\end{definition}
Again, for now, one can think of $n$ as being a multiple of $k$ without loss of generality. We will obtain AMQ-learners that do not require this.




To construct the learning algorithm, we will use a two-step process: first, we will construct a \textit{weak} agnostic learner, and then we will apply a distribution-specific agnostic boosting techniques, such as by \cite{feldman2009distribution} or \cite{kanade2009potential}.

\subsection{Weak Learning}

For \textit{weak} agnostic learning, we use a standard notion introduced by \cite{kalai2008agnostic} , because this type of weak learner can be subsequently used as an oracle in the boosting algorithms of \cite{feldman2009distribution, kanade2009potential}. Basically, the weak agnostic learner is required to recover some of the advantage that exists for the optimal concept $c \in \cC$. 

\begin{definition}[Weak agnostic learning]
    For $0 < \alpha \le \gamma < 1$, a $(\gamma, \alpha)$-weak agnostic learner for $\cC$ is an algorithm $\cA$ that (when receiving some oracle access to a target concept $\cD$) outputs a hypothesis $h: \{0,1\}^n \rightarrow \{-1,1\}$ such that 
    \[
        \max_{c \in \cal{C}} \Ex{(x,y)\sim \cD}{c(x)y} \ge \gamma \implies \Pr_{\cA}\left[\Ex{(x,y)\sim \cD}{h(x)y} \ge \alpha\right] \ge 2/3
    \]
\end{definition}

Our AMQ-learner works for the following class of distributions, which we term $t$-products. We use the notation $a||b$ to denote string concatenation of $a$ and $b$.

\begin{definition}[$t$-product distribution]
    Assume that $t \in \mathbb{Z}$ divides $n \in \mathbb{Z}$. A $t$-product distribution $\rho = \langle (\mu_i)_{i \in [t]}, \sigma \rangle$ is defined by $t$ distributions $(\mu_i)_{i \in [t]}$ each over $\{0,1\}^{n/t}$, and a permutation $\sigma$ over $[n]$. For $z \in \{0,1\}^{n}$, let $\sigma(z) := z_{\sigma(1)}||z_{\sigma(2)}||\cdots ||z_{\sigma(n)}$. To sample $\rho$, 
    \begin{enumerate}

        \item For each $i \in [t]$, sample $x_i \sim \mu_i$.
        \item Output $\sigma(x_1||\cdots ||x_t)$.
    \end{enumerate}
    Observe that a $t$-product distribution is defined to sample from $\{0,1\}^{n}$, since we assume that $t$ divides $n$.
    We denote by $\Delta^t_\sigma$ the class of $t$-product distributions that apply the permutation $\sigma$.
\end{definition}

The assumption that $t$ divides $n$ in the definition of $t$-products will also not sacrifice any generality in the ensuing AMQ-learner.
Rephrased now more formally, the end-goal of this subsection is to prove the following theorem.

\begin{theorem}\label{mqWeakAg}
There exists a $(\gamma, \gamma \cdot 2^{-c{2^k}-k})$-weak agnostic learner for $\dCC{k}{c}$. The learner uses membership queries and learns over any $\rho \in \Delta^{k}_\sigma$ (for some fixed $\sigma$), and has
\begin{itemize}
    \item query complexity $q := 2^{n-s(n)}$, for $s(n) = n/k-\log(k) +4(\log(\gamma) - c2^k +k))$
    \item time complexity $t := O(q)$.
    
\end{itemize}
\end{theorem}

First, we will prove an easier version, and then extend it via simple lemmas proved after. 

\begin{definition}[Boolean function correlation]
    Define $\corr{\Lambda} := \max_{h \in \Lambda}|\ex{f(x) \cdot h(x)}|$, where $x$ is sampled uniformly at random from the domain. Define $\corrD{\Lambda}{\rho} := \max_{h \in \Lambda}|\ex{f(x) \cdot h(x)}|$, where $x \sim \rho$.
\end{definition}

\begin{theorem}\label{mqWeakAgEasy}
Let $f :\{0,1\}^n \rightarrow \{-1,1\}$. Given the permutation $\sigma$, there exists an oracle algorithm $\cA$ with the guarantee that if $\corrD{\dCC{k}{c}}{\rho} \ge \gamma$, then for any $\rho \in \Delta^{k}_\sigma$, $\cA^{f,\rho}$ (with query access to $f$ and sampling access to $\rho$) outputs $h: \{0,1\}^n \rightarrow \{-1,1\}$ such that 
\[
\Pr_{\cA}\left[\Pr_{z \sim \rho}\left[{h(z) = f(z)}\right] \ge 1/2+\alpha\right] \ge 2/3
\]
with:
\begin{itemize}
    \item $\alpha := \Omega(\gamma \cdot 2^{-c{2^k}-k})$,
    \item query complexity $q := 2^{n-s(n)}$, for $s(n) = n/k-\log(k) +4(\log(\gamma) - c2^k +k))$
    \item time complexity $t := O(q)$.
\end{itemize}
\end{theorem}
\begin{proof}

To prove the theorem, we will:
\begin{enumerate}
    \item Present an algorithm $\cP^{f, \rho}$ that nearly witnesses the statement.
    \item We will then argue that establishing advantage, query and time complexity parameters for $\cP^{f, \rho}$ similar to those desired by Theorem \ref{mqWeakAgEasy} suffices to prove Theorem \ref{mqWeakAgEasy} (Lemma \ref{redux2Pred}).
    \item Finally, we will prove that $\cP^{f, \rho}$ does satisfy those parameters (Lemmas \ref{adv}, \ref{queryC} and \ref{timeC}). 
\end{enumerate}

Lemmas \ref{redux2Pred}, \ref{adv}, \ref{queryC} and \ref{timeC} suffice to prove Theorem \ref{mqWeakAgEasy}.
\end{proof}

\textbf{Notation.} For any string $z \in \{0,1\}^*$ denote the bit-wise negation of $z$ by $\bar{z}$. For the $2 \times n$ table $\mathbf{B}$ with entries $x^0_1, x^1_1 \cdots x^0_n, x^1_n \in \{0,1\}^m$, use $\mathbf{B}|_z \in \{0,1\}^{mn}$ to denote the concatenation of strings $x^{z_1}_1x^{z_2}_2\cdots x^{z_n}_n$. For a ``lookup'' table $\mathbf{T}$, let $\mathbf{T}[i]$ (for $i \in \{0,1\}^m$) denote the value stored at location $i$ of $\mathbf{T}$. For shorthand, let $k:=k(n), c:=c(n)$.

\begin{algorithm}[H]
  \caption{$\cP^{f, \rho}$}
  \label{P1}
  \begin{algorithmic}[1]
    \State \textbf{Input}: $k,n \in \nat$. A description of $\sigma$. Query access to $f$, sample access to $\rho$.
    \LComment{Begin preprocessing.}
    \State Sample $x^0,x^1 \sim \rho$.
    \State Apply $x^0 \gets \sigma^{-1}(x^0)$ and $x^1 \gets \sigma^{-1}(x^1)$
    \State Split $x^0,x^1$ into $x^0_1, x^1_1 \cdots x^0_k, x^1_k \in \{0,1\}^{n/k}$, which are $k$ $n/k$-bit blocks so that each $x^0_i,x^1_i$ is sampled according to the $i^{th}$ distribution that makes up $\rho$.
    \State Choose uniformly random string $b \in \{0,1\}^{k}$. \label{choiceb}
    \State Choose uniformly random string $r \in \{-1,1\}^{2^k}$.
    Partially fill a $2 \times k$ table $\mathbf{B}$, such that $\mathbf{B}[b_i,i] = x^{b_i}_i$. Fill other entries with $\{*\}^{n}$.
    \State \textbf{For} every string $a < \bar{b}$ 
    \begin{quote}
        Viewing $\mathbf{B}|_a$ as a partial assignment $z^*$ of $n$ bits, \textbf{query} $f$ on all $n$-bit points consistent with $\sigma(z^*)$.\\
        Place the query-label pairs in a lookup table $\mathbf{T}$.
    \end{quote} 
    \LComment{End preprocessing.}
    \State \textbf{Generate and output} a circuit $h$ (with the random string $r \in \{-1,1\}^{2^k}$ hard-wired), according to the following template:
    \begin{quote}
    On input $z \in \{0,1\}^n$,\\
    Place the $i^{th}$ out of $k$, $n/k$-bit blocks of $z$ inside $\mathbf{B}[\bar{b}_i,i]$.\\
    Compute the values
    \[
    v = \prod_{a < \bar{b}} \mathbf{T}[\mathbf{B}|_a] \ \ \text{and} \ \ v' = \prod_{a \ge \bar{b}} r_a
    \]
    Output $v\cdot v' \cdot r_{\bar{b}}$
    \end{quote}
  \end{algorithmic}
\end{algorithm}

\begin{lemma}\label{redux2Pred}
Suppose that $\cP$ makes at most $q := q(n)$ queries while running in time $t:= t(n)$, and that whenever $\corrD{\dCC{k}{c}}{\rho} \ge \gamma$, it holds that
    \begin{align}\label{eqRandAdv}
        \Pr_{R,z \sim \rho}\Big[h(z) = f(z) : h \leftarrow \cP^{f, \rho}\Big] \ge \frac{1}{2} + \Omega(\alpha)
    \end{align}
for randomness $R = (r,b,x^0_1,x^1_1, \cdots x^0_k,x^1_k)$ of $\cP$. Then there exists a randomized oracle algorithm $\cA^{f, \rho}$ with the guarantee that if $\corrD{\dCC{k}{c}}{\rho} \ge \gamma$, then $\cA^{f, \rho}$ outputs $h: \{0,1\}^n \rightarrow \{-1,1\}$ such that 
\[
\Pr_{\cA}\left[\Pr_{z \sim \rho}\left[h(z) = f(x)\right] \ge \frac{1}{2}+ \Omega(\alpha)\right] \ge 2/3
\]
with:
\begin{itemize}
    \item $\alpha := \Omega(\gamma \cdot 2^{-c{2^k}-k})$,
    \item query complexity $q' := O(\alpha^{-4})\cdot q$,
    \item time complexity $t' := O(q')$.
\end{itemize}
\end{lemma}

\begin{lemma}[Chernoff Bound, cf. Theorem 2.1 \cite{janson2011random}]\label{chernoff}
    Let $X \sim {\sf Bin}(m,p)$ and $\lambda = m \cdot p$. For any $t \ge 0$,
    \[
    \Pr[|X-\Ex{}{X}| \ge t] \le \exp\left(\frac{-t^2}{2(\lambda + t/3)}\right)
    \]
\end{lemma}

\begin{proof}[Sketch.]
    The statement follows by a constructive averaging argument. First, consider that \eqref{eqRandAdv} implies that
    \begin{align}
            \Pr_{R}\bigg[\Pr_{z}\Big[h(z) = f(z) : h \leftarrow \cP^{f, \rho}\Big] \ge \frac{1}{2} + \Omega(\alpha)\bigg] \ge \Omega(\alpha)
    \end{align}
    by standard averaging (see \cite{arora2009computational} for reference).

    Therefore, viewing $\cP^{f, \rho}$ as a distribution over hypotheses $h$, it follows that $\cA^{f, \rho}$ exists by  efficiently constructing a ``good'' hypothesis by randomized trial-and-error. We may sample $O(\alpha^{-2})$ candidate hypotheses in parallel using $\cP^{f, \rho}$ and then compare each to the concept by checking random examples. By application of Lemma \ref{chernoff},  there will be a ``good'' hypothesis with constant probability, and $O(\alpha^{-2})$ examples will be enough to check that each circuit with good enough accuracy is indeed good enough, with constant probability. In total, with  $O(\alpha^{-4})$ random samples, $\cA^{f, \rho}$ uses $\cP^{f, \rho}$ to output $h: \{0,1\}^n \rightarrow \{-1,1\}$ such that 
    \[
    \Pr_{\cA}\left[\Pr_{z \sim \rho}{h(z) = f(x)} \ge \frac{1}{2}+ \Omega(\alpha)\right] \ge 2/3
    \]


\end{proof}

We will now establish that the properties of $\cP^{f, \rho}$ are as desired.

\begin{lemma}[Advantage]\label{adv}
    It holds that
    \begin{align}
        \Pr_{R,z}\Big[h(z) = f(z) : h \leftarrow \cP^{f, \rho}\Big] \ge \frac{1}{2} + \Omega(\alpha)
    \end{align}
for randomness $R = (r,b,x^0_1,x^1_1, \cdots x^0_k,x^1_k)$ of $\cP$.
\end{lemma}

\begin{proof}
    To prove the lemma, we will use the fact that the $k$-party norm can help distinguish random functions from functions that correlate with $k$-party protocols, together with a hybrid argument.

    We will define the following distributions.
    \begin{itemize}
        \item Let $Q_a$ be the distribution over the value at location $\mathbf{B}|_a$ inside lookup table $\mathbf{T}$ inside $\cP^{f, \rho}$. This distribution is over the randomness of $z,b,x^0_1,x^1_1, \cdots x^0_k, x^1_k$.
        \item Let $Q = (Q_a)_{a \in \{0,1\}^{k}}$ denote the joint distribution over $Q_a$ for all possible strings $a \in \{0,1\}^{k}$.
        \item For $b^* \in \{0,1\}^k$ (representing an integer in $[2^k]$), define the distribution $H_{a,b^*}$ as 
            \[H_{a,b^*} = 
                \begin{cases}
                Q_a \text{ when $a < b^*$} \\
                U_1 \text{ otherwise}
                \end{cases}\\
            \]
        \item Define the distribution $H_{b^*}$ over $\{-1,1\}^{2^k}$ as the joint distribution over $(H_{a,b^*})_{a \in \{0,1\}^{k}}$.
    \end{itemize}


    Now, observe that for $H_{\{1\}^k}$, we have that
    \begin{align*}
    \Ex{\sigma \sim H_{\{1\}^k}}{\prod_{i \in [2^k]}\sigma_i} = \Ex{\sigma \sim Q}{\prod_{i \in [2^k]}\sigma_i}  &= \Ex{\substack{z, r, \\x^0_1,x^1_1 \\ \cdots \\ x^0_k,x^1_k}}{\prod_{a \in \{0,1\}^{k}} \mathbf{T}[\mathbf{B}|_a]}
    \end{align*}
    Let $z^{(i)}$ denote the $i^{th}$ (out of $k$) block of $n$ bits of $z$, in order from most significant bit to least. Then,
    \begin{align}\label{eq:rhs}
    \Ex{\substack{z, r, \\x^0_1,x^1_1 \\ \cdots \\ x^0_k,x^1_k}}{\prod_{a \in \{0,1\}^{k}} \mathbf{T}[\mathbf{B}|_a]} &= \Ex{\substack{z, r, \\x^0_1,x^1_1 \\ \cdots \\ x^0_k,x^1_k}}{\prod_{a_1, \cdots a_k \in \{0,1\}} f(x^{a_1}_1, \cdots x^{a_k}_k) \mid z^{(i)} = x^1_i \ \forall \ i \in [k]}
    \end{align}
Thus, when $z \sim \rho$, the quantity on the right-hand-side of \eqref{eq:rhs} is, by definition, equivalent to $R_k(f \circ \rho)$. That is, the $k$-party norm of the composition of $f$ with a sampler for the distribution $\rho \in \Delta^{k}_\sigma$. Therefore, since $\rho$ has the $k$-product structure, our assumption implies that
\begin{align*}
     \Ex{\sigma \sim H_{\{1\}^k}}{\prod_{i \in [2^k]}\sigma_i}= R_k(f \circ \rho) = R_k(f) \ge \gamma \cdot 2^{-c2^k}
\end{align*}
Again, the equality between $R_k(f \circ \rho) = R_k(f)$ holds because $\rho$ is a $k$-product distribution, so if $\corrD{\dCC{k}{c}}{\rho}\ge \gamma$, then $\corrClass{f \circ \rho}{\dCC{k}{c}}\ge \gamma$, because computation of each of the $k$ components of $\rho$ before $f$ can be done by the appropriate party by with no added communication between the parties, or loss in success probability. Additionally, we are using the fact that we have defined $\dCC{k}{c}$ so that for each function in the class, and every partition of the inputs, there is a protocol that transmits at most $c$ bits.

On the other hand, 
\begin{equation*}
    \Ex{\sigma \sim H_{\{0\}^k}}{\prod_{i\in [2^k]}\sigma_i} = \Ex{\sigma \sim \{-1,1\}^{2^k}}{\prod_{i\in [2^k]}\sigma_i}  =0
\end{equation*}

Proceeding by a hybrid argument, it is then the case that 
for random hybrid neighbors $H_{j}, H_{j+1}$ ($j\in \{0,1\}^k$, with $+$ indicating integer addition),

\begin{align}\label{neighbor2}
    \mathop{\mathbb{E}}_{j \sim \{0,1\}^k}\left[\Ex{\sigma \sim H_{j+1}}{\prod_{i\in [2^k]}\sigma_i = 1} - \Ex{\sigma \sim H_j}{\prod_{i\in [2^k]}\sigma_i =1}
    \right] \ge \gamma \cdot 2^{-c2^k} \cdot 2^{-k}
\end{align}
Observe that, for $z \sim \rho$, the circuit $h$ output by $\cP^{f, \rho}$, by definition, outputs the value 
\[
h(z) = v\cdot v'\cdot r_{\bar{b}} = r_{\bar{b}} \prod_{i \in [2^k]}\sigma_i
\]
where $\sigma \sim H_{\bar{b}}$ and $\bar{b} \sim U_{k}$. 
Hence, we interpret $h(z)$ as a prediction, where $r_{\bar{b}}$ is the ``guess bit.'' 

To ease notation, let $R = (z,r, b,x^0_1,x^1_1 \cdots x^0_k, x^1_k)$. Conditioning on correctness of the guess bit,
\begin{align*}
    \Pr_{R}\big[h(z) = f(z) \mid r_{\bar{b}} = f(z)\big]
    &= \Pr_{R}\big[h(z) = f(z) \mid r_{\bar{b}} = f(z)\big]
      \cdot \Pr_{R}[r_{\bar{b}} = f(z)] \\
    &\ \ \ \ + \Pr_{R}\left[h(z) = f(z)
      \mid r_{\bar{b}} \not= f(z)\right] \cdot \Pr_{R}[r_{\bar{b}} \not= f(z)]
\end{align*}
Then by making the appropriate substitution, we obtain
\begin{align*}
    \Pr_{R}\big[h(z) = f(z) \mid r_{\bar{b}} = f(z)\big] &= \frac{1}{2}\Pr_{R}\left[r_{\bar{b}} \prod_{i \in [2^k]}\sigma_i = f(z) \mid r_{\bar{b}} = f(z)\right]\\
    &\ \ \ + \frac{1}{2}\Pr_{R}\left[r_{\bar{b}} \prod_{i \in [2^k]}\sigma_i = f(z) \mid r_{\bar{b}} \not= f(z)\right]
\end{align*}

Indeed, when $\prod_{i \in [2^k]}\sigma_i = 1$ ($\sigma \sim H_{\bar{b}}$), this means that $h(z)=r_{\bar{b}}$. 
The case analysis then follows:
\begin{align*}
    \Pr_R\big[
    h(z) = f(z)\big] &=
              \frac{1}{2}\left(\Pr_R\left[\prod_{i \in [2^k]}\sigma_i = 1 \mid r_{\bar{b}} = f(z)\right]
              + \Pr_R\left[\prod_{i \in [2^k]}\sigma_i = -1 \mid r_{\bar{b}} \not= f(z)\right]\right)\\
            &= \frac{1}{2}
              + \frac{1}{2}\left(-\Pr_R\left[\prod_{i \in [2^k]}\sigma_i = -1 \mid r_{\bar{b}} = f(z)\right]
              + \Pr_R\left[\prod_{i \in [2^k]}\sigma_i = -1 \mid r_{\bar{b}} \not= f(z)\right]\right)
  \end{align*}
  By conditioning we know that:
  \begin{equation*}
    \Pr_R\left[\prod_{i \in [2^k]}\sigma_i = -1\right] =
    \frac{1}{2} \Pr_R\left[\prod_{i \in [2^k]}\sigma_i = -1 \mid r_{\bar{b}} = f(z)\right]
    + \frac{1}{2}\Pr_R\left[\prod_{i \in [2^k]}\sigma_i = -1\mid r_{\bar{b}} \not= f(z)\right]
  \end{equation*}
  rearranging the terms, we get:
  \begin{equation*}
    \frac{1}{2}\Pr_R\left[\prod_{i \in [2^k]}\sigma_i = -1 \mid r_{\bar{b}} \not= f(z)\right] =
    \Pr_R\left[\prod_{i \in [2^k]}\sigma_i = -1\right]
    - \frac{1}{2} \Pr_R\left[\prod_{i \in [2^k]}\sigma_i = -1 \mid r_{\bar{b}} = f(z)\right]
  \end{equation*}
  We thus conclude that:
  \begin{align*}
    \Pr_R[h(z) = f(z)] &=
    \frac{1}{2} + \underbrace{\Pr_R\left[\prod_{i \in [2^k]}\sigma_i = -1\right]}_{(1)} - \underbrace{\Pr_R\left[\prod_{i \in [2^k]}\sigma_i = -1 \mid r_{\bar{b}} = f(z)\right]}_{(2)}
  \end{align*}
The term $(1)$ corresponds the the case of taking the probability of the parity of a sample from $H_{\bar{b}}$ and the result being -1, while the term $(2)$ corresponds analagously to the case of taking the parity of a sample from $H_{\bar{b}+1}$.

\begin{align*}
    \Pr_R[h(z) = f(z)] &= \frac{1}{2} + \left(1-\Pr_R\left[\prod_{i \in [2^k]}\sigma_i =1\right]\right) - \left(1-\Pr_R\left[\prod_{i \in [2^k]}\sigma_i = 1 \mid r_{\bar{b}} = f(z)\right]\right)\\
    &=  \frac{1}{2} -\Pr_R\left[\prod_{i \in [2^k]}\sigma_i = 1\right] +\Pr_R\left[\prod_{i \in [2^k]}\sigma_i = 1 \mid r_{\bar{b}} = f(z)\right]
\end{align*}
Therefore, by equation \eqref{neighbor2},
  \begin{align*}\label{bound}
    \Pr\big[h(z) = f(z)\big] \ge \frac{1}{2}+\gamma \cdot 2^{-c2^k} \cdot2^{-(k+1)}
  \end{align*}
\end{proof}

We will now prove that the query complexity of $\cP^{f, \rho}$ is as needed.

\begin{lemma}[Query complexity of $\cP^{f, \rho}$]\label{queryC}
    $\cP^{f, \rho}$ makes at most $2k \cdot {2^{n-n/k}}$ queries.
\end{lemma}
\begin{proof}
    We can (roughly) bound the number of queries $q$ as a function of $n,k$ as follows. 
    Observe that, for any random choices $X = x_1^0,x_1^1, \cdots x_k^0,x_k^1, b$ of $\cP^{f, \rho}$, the queries by $\cP^{f, \rho}$ are a subset of union of the $k$ ``block restrictions'' 
        \[
        \{\underbrace{** \cdots *}_{k-1 \ \textrm{times}}x_k^b\} \cup
        \{\underbrace{** \cdots *}_{k-2 \ \textrm{times}}x_{k-1}^b*\} \cup \cdots \cup \{x_1^b\underbrace{** \cdots *}_{k-1 \   \textrm{times}}\}
        \]
    Here, $\{{** \cdots *} \ x_k^b\}$ is notation for the $k^{th}$ block restriction which is the set of $n$-bit points that are consistent with the $k^{th}$ block set to $x_k^b$. We call $x_k^b$ the \textit{representative} of $\{{** \cdots *} \ x_k^b\}$ .
    It is easy to see that the size of the union set of these $k$ ``block'' restrictions is upper bounded by $k{2^{n-n/k}}$.
    Now, since $\cP^{f, \rho}$ is supposed to repeat the queries for all $a < \bar{b}$, actually we can see that we need the union set 
    \[
        \bigcup\limits_{b \in \{0,1\}^k}\{\underbrace{** \cdots *}_{k-1 \ \textrm{times}}x_k^{b_k}\} \cup
        \{\underbrace{** \cdots *}_{k-2 \ \textrm{times}}x_{k-1}^{b_{k-1}}*\} \cup \cdots \cup \{x_1^{b_1}\underbrace{** \cdots *}_{k-1 \   \textrm{times}}\}
        \]
    This amounts to at most $2k \cdot {2^{n-n/k}}$ queries.
\end{proof}

Finally, we bound the time complexity of $\cP^{f, \rho}$.

\begin{lemma}[Time complexity]\label{timeC}
    $\cP^{f, \rho}$ runs in time $O(2k \cdot {2^{n-n/k}})$.
\end{lemma}
\begin{proof}
    It suffices to observe that the complexity of $\cP^{f, \rho}$ is dominated by the generation and memorization of the queries, as well as the construction of the hypothesis circuit. Each of these tasks is completed in time $O(2k \cdot {2^{n-n/k}})$, so the lemma follows.
\end{proof}

\paragraph{Towards Theorem \ref{mqWeakAg}. }

In order to get a $(\gamma, \gamma \cdot 2^{-c{2^k}-k})$-weak agnostic learner for $\dCC{k}{c}$, we need to generalize the oracle algorithm $\cA^{f, \rho}$ we got by Theorem \ref{mqWeakAgEasy} to work for arbitrary (probabilistic) concepts.

To do this, we first show that $\cA^{f, \rho}$ queries $f$ entirely at distinct points. In most time complexity regimes this would be trivial, since $f$ is a function and no new information is gained by making a duplicate query. However, when time complexity is just barely nontrivial, we need to take care that we can actually efficiently build a lookup table with no duplicates. 

\begin{lemma}\label{noDupes}
    $\cA^{f, \rho}$ need not make any duplicate queries.
\end{lemma}
\begin{proof}
    Let us first inspect the queries made by the first run of $\cP^{f, \rho}$. 
    We observe that, when querying the union of block restrictions 
    \[
        \bigcup\limits_{b \in \{0,1\}^k}\{\underbrace{** \cdots *}_{k-1 \ \textrm{times}}x_k^{b_k}\} \cup
        \{\underbrace{** \cdots *}_{k-2 \ \textrm{times}}x_{k-1}^{b_{k-1}}*\} \cup \cdots \cup \{x_1^{b_1}\underbrace{** \cdots *}_{k-1 \   \textrm{times}}\}
    \]
    it is possible to avoid duplicate queries in an online way: start with $\{** \cdots *x_k^{0}\}$ and $\{** \cdots *x_k^{1}\}$, and then, for $\{** \cdots x_{k-1}^{0} *\}$ and $\{** \cdots x_{k-1}^{1}*\}$ skip any element that has postfix $x_k^{1}$ or $x_k^{0}$ (for each element, this can be checked in time $O(kn)$). Repeat this process for each remaining pair of block restrictions to avoid duplicates, while maintaining time complexity $O(2^{n-n/k+\log(k)})$.

    We need to consider what happens in each subsequent run of $\cP^{f, \rho}$ executed by $\cA^{f, \rho}$. 
    Again, by tracking the list of representatives chosen in previous runs, duplicates can be efficiently avoided. There are $O(\alpha^{-4})$ runs, so at most the complexity of scanning the list of previous representatives is $O(\alpha^{-4}\cdot kn)$. Therefore time complexity of $\cA^{f, \rho}$ remains asymptotically the same even after incorporating the duplicate-scanning procedure.
\end{proof}

We are finally in position to argue that we have a weak agnostic learner for $\dCC{k}{c}$. The main point is that the learning algorithm we have so far considered, $\cA^{f, \rho}$, does not consider probabilistic concepts. Probabilistic concepts are not functions, as they may label the same point differently depending on some internal randomness.

\begin{theorem}[Theorem \ref{mqWeakAg} restated]
There exists a $(\gamma, \gamma \cdot 2^{-c{2^k}-k})$-weak agnostic learner for $\dCC{k}{c}$. The learner uses membership queries and learns over any $\rho \in \Delta^{k}_\sigma$ (for some fixed $\sigma$), and has
\begin{itemize}
    \item query complexity $q := 2^{n-s(n)}$, for $s(n) = n/k-\log(k) +4(\log(\gamma) - c2^k +k))$
    \item time complexity $t := O(q)$.
\end{itemize}
\end{theorem}
\begin{proof}
    For weak agnostic learning, the target concept is an arbitrary distribution over $\{0,1\}^n \times \{-1,1\}$. This is more general than the idea of a concept being an arbitrary Boolean function that labels points, because now the concept may be probabilistic (i.e., not the same label at each point, every time). We need to argue that, for $0 < \gamma < 1$, a there is an algorithm that outputs a hypothesis $h: \{0,1\}^n \rightarrow \{-1,1\}$ such that 
    \[
        \max_{\pi \in \dCC{k}{c}} \Ex{(x,y)\sim \cD}{\pi(x)y} \ge \gamma \implies \Pr_{\cA}\left[\Ex{(x,y)\sim \cD}{h(x)y} \ge \gamma \cdot 2^{-c{2^k}-k} \right] \ge 2/3
    \]

    The algorithm from Theorem \ref{mqWeakAgEasy} suffices. This is because, by Lemma \ref{noDupes}, the target concept might as well have been deterministic -- there are no duplicate queries. Hence, Theorem \ref{mqWeakAgEasy}, in combination with Lemma \ref{noDupes}, implies that for $0 < \gamma < 1$, we get an algorithm $\cA^{\cD}$ that uses membership query and random example access to $\cD$ to simulate both $f$ and the marginal distribution $\rho$, and outputs $h: \{0,1\}^n \rightarrow \{-1,1\}$ such that 
    \[
        \max_{\pi \in \dCC{k}{c}} \Ex{(x,y)\sim \cD}{\pi(x)y} \ge \gamma \implies \Pr_{\cA}\left[\Ex{(x,y)\sim \cD}{h(x)y} \ge \gamma \cdot 2^{-c{2^k}-k} \right] \ge 2/3
    \]
    
\end{proof}

Now that we have a weak agnostic learning algorithm, we can use distribution-specific agnostic boosting algorithms to obtain fully agnostic learning (Theorem \ref{mqWeakAg}). Distribution-specific boosting algorithms convert weak AMQ-learners into strong AMQ-learners, by deliberately modifying the labels of queries. This contrasts with other types of boosting algorithms, which instead modify the distribution over unlabelled examples. In our case, we need to use a distribution-specific boosting algorithm, because our weak learning algorithm does not work over all distributions. 

The following theorem is a restatement of distribution-specific boosting, due to \cite{feldman2009distribution} and \cite{kanade2009potential}.

\begin{theorem}[Distribution-specific agnostic boosting]\label{thm:ds-boosting}
There exists an algorithm {\sf boost} that for every concept class $\cC$
and distribution $\rho$ over $\{0,1\}^{n}$, given an $(\gamma, \alpha)$-weak agnostic learning algorithm $\cA$ for $\cC$ over $\rho$, agnostically learns $\cC$ over $\rho$. Further, {\sf boost} invokes $\cA$ $O(\alpha^{-2})$ times and runs in time $T \cdot
\poly(\alpha^{-1}, \epsilon^{-1})$, where $T$ is the running time of $\cA$.
\end{theorem}

We combine this theorem with Theorem \ref{mqWeakAg}, to obtain the strong learning algorithm:

\begin{theorem}\label{mainThm}
    There exists an agnostic membership query learning algorithm for $\dCC{k}{c}$. The algorithm learns over any $\rho \in \Delta^{k}_\sigma$ (for some fixed $\sigma$), and has
\begin{itemize}
    \item query complexity $q := 2^{n-s(n)}$, for $s(n) = n/k-\log(k) +4(\log(\gamma) - c2^k +k))$
    \item time complexity $t := O(q) \cdot \poly(\epsilon^{-1})$.
\end{itemize}
Here, $\gamma := \opt{\dCC{k}{c}}$. 
\end{theorem}

\begin{proof}
    The theorem statement follows immediately from Theorem \ref{thm:ds-boosting} and \ref{mqWeakAg}. To remove the assumption that $k$ dives $n$, simply consider that if it does not, then  $k- (n \text{ mod } k)$ meaningless variables can be added to the input of the function so that $k$ divides $n$, but the functionality is unchanged. Therefore this will not affect the functionality of the learning algorithm, and the complexity is increased by at most a constant factor.
\end{proof}

From here, we derive as immediate corollaries:

\begin{theorem}\label{thm:AgnosticSYM circuits}
    There exists an agnostic membership query learning algorithm for circuits consisting of $t_1$ gates, where each gate can compute a $\sym^+$ circuit of size $t_2$, and degree $k-1$ (i.e., the class $\sym^+{\text -}{\sf Ckt}[k-1, t_2, t_1]$). The algorithm learns over any $\rho \in \Delta^{k}_\sigma$ (for some fixed $\sigma$), and has
\begin{itemize}
    \item query complexity $q := 2^{n-s(n)}$, for $s(n) = n/k-\log(k) +4(\log(\gamma) - 2^kt_1\log(t_2) +k))$
    \item time complexity $t := O(q) \cdot \poly(\epsilon^{-1})$.
\end{itemize}
Here, $\gamma := \opt{\sym^+{\text -}{\sf Ckt}[k-1, t_2, t_1]}$.
\end{theorem}

\begin{proof}
    $\sym^+$ circuits of size $t_2$ and degree $k-1$ are well-known to be a subset of $\Pi[k,\log(t_2)]$ (see e.g. Lemma 4 of \cite{haastad1991power}). Then, by simulating a circuit of $t_1$ gates, each computing $\sym^+$ circuits of size $t_2$ and degree $k-1$, we can still get a protocol in the class $\Pi[k,t_1\log(t_2)]$ (see also \cite{nisan1993communication}, section 4). Recall, this means that there is a $k$-party deterministic protocol that communicates $t_1\log(t_2)$ bits \textit{for any} partition of the input into $k$ parts. Then, the statement follows immediately from Theorem \ref{mainThm}. 
    

\end{proof}

\begin{theorem}\label{thm:AgnosticSYM-DTs}
    There exists an agnostic membership query learning algorithm for decision trees of depth $d$, with node queries computed by $\sym^+$ circuits of size $t$, and degree $k-1 \ge 1$ (i.e., the class $\sym^+{\text -}{\sf Dt}[k-1, t, d]$). The algorithm learns over any $\rho \in \Delta^{k}_\sigma$ (for some fixed $\sigma$), and has
\begin{itemize}
    \item query complexity $q := 2^{n-s(n)}$, for $s(n) = n/k-\log(k) +4(\log(\gamma) - 2^kd\log(t) +k))$
    \item time complexity $t := O(q) \cdot \poly(\epsilon^{-1})$.
\end{itemize}
Here, $\gamma := \opt{\sym^+{\text -}{\sf Dt}[k-1, t, d]}$.
\end{theorem}

\begin{proof}
    $\sym^+$ circuits of size $t_2$ and degree $k-1$ are well-known to be a subset of $\Pi[k,\log(t_2)]$ (see e.g. Lemma 4 of \cite{haastad1991power}). Then, by simulating the decision tree of depth $d$, with $\sym^+$ queries of size $t$ and degree $k-1$, we can still get a protocol in the class $\Pi[k,d\log(t)]$ (see also \cite{nisan1993communication}, section 4). Then, the statement follows immediately from Theorem \ref{mainThm}.

\end{proof}

\section{Learning Randomized $k$-Party NOF Communication Protocols}

In this section, we will convert the agnostic learning algorithm for $\dCC{k}{c}$ into an agnostic learning algorithm for the class of functions computable by \textit{randomized} $k$-party communication protocols of cost $c$, $\rCC{k}$. The set $\rCC{k}$ trivially contains $\dCC{k}{c}$, and may be more powerful. As a case in point, PTFs of degree $k$ are not known to be computed by $\dCC{k}{c}$, but \textit{can} be computed by $\rCC{k}$.

\begin{definition}[Randomized $\dCC{k}{c}$]
The randomized $k$-party communication model is the same as the deterministic model, except we now allow the protocol to depend on random bits. Therefore, we allow the protocol to err in its output. The probability of error of a randomized protocol is $\epsilon$ if for every input to the function $f$, the protocol errs in outputs with probability at most $\epsilon$. We denote by $\randCC{k}$ the class of $k$-party randomized protocols that, for any partition of the $n$ variables into $k$ length $n/k$ parts, transmit at most $c$ bits and err with probability at most $1/2- \alpha$. Additionally, we define $\RANDCC{k}{c}{1/6} := \rtCC{k}{c}$ for shorthand.
\end{definition}
For the sake of simplicity, this paper uses only the ``public coin'' version of randomized communication complexity. Namely, the parties all share a string of random bits. Again, we assume for now that $k$ divides $n$.

To get started, we prove the following lemma which converts correlation with randomized protocols to correlation with deterministic protocols.

\begin{lemma}\label{rCC2dCC}
    Let $f: \{0,1\}^{n} \rightarrow \{-1,1\}$. If $\corr{\rCC{k}} \ge \alpha$, then $\corr{\dCC{k}{O(c\log(\alpha^{-1}))}} \ge \alpha/2$.
\end{lemma}
\begin{proof}
    Guaranteed by the condition of the lemma is the existence of $g: \{0,1\}^n \rightarrow \{-1,1\}$ such that $g\in \rCC{k}$ and $\corr{g} \ge \alpha$. Without loss of generality assume $g$ is the lexicographically first of all functions that satisfy $g\in \rCC{k}$ and $\corr{g} \ge \alpha$. Let $\pi$ be the protocol that witnesses the inclusion $g\in \rCC{k}$, and let $\pi^m$ be the direct-majority protocol of $\pi$ so as to reduce error to $1-\alpha/2$. 
    We thus have that $g \in \RANDCC{k}{O(c\log(\alpha^{-1}))}{\alpha/2}$ and $\corr{g} \ge \alpha$. 
    
    Now, here is a deterministic protocol $\pi^*$ that witnesses $\corr{\dCC{k}{O(c\log(\alpha^{-1}))}} \ge \alpha/2$. First, both parties find $g$, and compute the set $T$ of inputs that $f$ and $g$ agree on. Then, they convert $\pi^m$ to a distributional protocol $\tilde{\pi}^m$ for the uniform distribution over $T$ (possible by an averaging argument), and run $\tilde{\pi}^m$ and output the result. Because $2^{-n}|T| \ge 1/2+\alpha/2$, we get that $\corr{\pi^*} \ge 2(1/2+\alpha/2)(1-\alpha/2)-1 \ge (1+\alpha)(1-\alpha/2)-1 \ge \alpha/2$.
\end{proof}

Now we can obtain a weak agnostic learning algorithm for $\rCC{k}$, as long as the 
distribution over unlabelled examples is uniform. Then boosting gives us a strong agnostic learning algorithm.

\begin{theorem}\label{thm:agnostic-rCC}
    There exists an agnostic membership query learning algorithm for $\rCC{k}$.
    The algorithm learns over the uniform distribution, and has
    \begin{itemize}
        \item query complexity $q := 2^{n-s(n)}$, for $s(n) = n/k-k-O(\log(\alpha^{-1}) + c{2^k}))$
        \item time complexity $t := O(q)$.
    \end{itemize}
    Here, $\gamma := \opt{\rCC{k}}$.
\end{theorem}

\begin{proof}
    We need to prove that there is a membership query algorithm $\cA$ that outputs a hypothesis $h: \{0,1\}^n \rightarrow \{-1,1\}$ such that, for any $\cD$ such that the distribution over unlabelled examples, $\rho$, is uniform, we have that
    \[
        \max_{c \in \rCC{k}} \Ex{(x,y)\sim \cD}{c(x)y} \ge \alpha \implies \Pr_{\cA}\left[\Ex{(x,y)\sim \cD}{h(x)y} \ge \gamma\right] \ge 2/3
    \]
    If we do so, then we can invoke boosting (Theorem \ref{thm:ds-boosting}) to get an agnostic learning algorithm for $\rCC{k}$. Thus, we prove the following:

    \begin{lemma}\label{mqWeakAgRandCC}
        There exists a $(\gamma, \gamma \cdot 2^{-O(c\log(\gamma^{-1})){2^k}-k})$-weak agnostic learner for $\rCC{k}$. The learner uses membership queries and learns over the uniform distribution over unlabelled examples, and has
        \begin{itemize}
            \item query complexity $q := 2^{n-s(n)}$, for $s(n) = n/k-k-O(\log(\gamma^{-1}) + c{2^k}))$
            \item time complexity $t := O(q)$.
        \end{itemize}
    \end{lemma}

    \begin{proof}
    By lemma \ref{rCC2dCC}, if 
    \[
    \max_{c \in \rCC{k}} \Ex{(x,y)\sim \cD}{c(x)y} \ge \gamma
    \]
    then we know that 
    \[
        \max_{c \in \dCC{k}{O(c\log(\gamma^{-1}))}} \Ex{(x,y)\sim \cD}{c(x)y} \ge \gamma/2
    \] 
    Then, by Theorem \ref{mqWeakAg}, we know that there is a $(\gamma, \gamma \cdot 2^{c{2^k}-k})$-weak agnostic learning algorithm for $\dCC{k}{c}$. Therefore, we conclude that there is an $(\gamma, \gamma \cdot 2^{-O(c\log(\gamma^{-1})){2^k}-k})$-weak agnostic learner $\cA_{wk}$ for $\rCC{k}$:
    \begin{align*}
        \max_{c \in \rCC{k}} \Ex{(x,y)\sim \cD}{c(x)y} \ge \gamma &\overset{{\text{(Lemma \ref{rCC2dCC})}}}{\implies}\\
        \max_{c \in \dCC{k}{O(c\log(\gamma^{-1}))}} \Ex{(x,y)\sim \cD}{c(x)y} \ge \gamma/2 &\overset{{\text{(Theorem \ref{mqWeakAg})}}}{\implies} \\
        \Pr_{\cA_{wk}}\left[\Ex{(x,y)\sim \cD}{h(x)y} \ge \gamma\right] \ge 2/3
    \end{align*}
\end{proof}

Finally, Theorem \ref{thm:ds-boosting} and Lemma \ref{mqWeakAgRandCC} combine to indicate that $\cA_{wk}$ can be boosted to a strong agnostic learning algorithm, that has query complexity $q := 2^{n-s(n)}$, for $s(n) = n/k-k-O(\log(\gamma^{-1}) + c{2^k}))$ and time complexity $t := O(q)$. Recall, from Theorem \ref{thm:ds-boosting}, the weak learner is invoked $O(\gamma^{-2}$ times), so that is where the $-O(\log(\gamma^{-1}) + c{2^k})$ additive term comes from.
\end{proof}

\subsection{Learning Circuits with Threshold Gates}

In this section, we will derive AMQ-learners for circuits with threshold gates, and decision trees that are allowed threshold queries. To do this, we use the following randomized communication protocols for such function representations, due to \cite{nisan1993communication}.

\begin{theorem}[\cite{nisan1993communication}]\label{thm:Nisan-upperbounds}
Let $f: \{0,1\}^n \rightarrow \{-1,1\}$ be a function.
    \begin{enumerate}
        \item If $f$ can be computed by a circuit consisting of $s$ gates, each computing a polynomial threshold function of degree $k$, then $f \in \rtCC{k+1}{O(sk^3\log(n)\log(sn))}$
        \item If $f$ can be computed by a decision tree of depth $d$, with each node computing a polynomial threshold function of degree $k$, then $f \in \rtCC{k+1}{O(dk^3\log(n)\log(dn))}$.
    \end{enumerate}
\end{theorem}

From this theorem, combined with Theorem \ref{thm:agnostic-rCC}, we obtain the following as corollaries. Recall ${\sf PT}{\text -}{\sf Ckt}[k,m]$ is the class of circuits consisting of at most $m$ polynomial threshold gates of degree $k$, and ${\sf PT}{\text -}{\sf Dt}[k,d]$ is the class of decision trees consisting of at most depth $d$, with internal nodes computing polynomial threshold gates of degree $k$.

\begin{theorem}\label{thm:agnostic-PTF}
    There exists an agnostic membership query learning algorithm for ${\sf PT}{\text -}{\sf Ckt}[k,m]$.
    The algorithm learns over the uniform distribution, and has
    \begin{itemize}
        \item query complexity $q := 2^{n-s(n)}$, for $s(n) = n/(k+1)-O(\log(\gamma^{-1}) +2^k \cdot m\log(n)\log(mn))$
        \item time complexity $t := O(q)$.
    \end{itemize}
    Here, $\gamma := \opt{{\sf PT}{\text -}{\sf Ckt}[k,m]}$. 
\end{theorem}
\begin{proof}
    Immediate, from Theorem \ref{thm:agnostic-rCC} and \ref{thm:Nisan-upperbounds}.
\end{proof}

\begin{theorem}\label{thm:agnostic-DT}
    There exists an agnostic membership query learning algorithm for ${\sf PT}{\text -}{\sf Dt}[k,d]$.
    The algorithm learns over the uniform distribution, and has
    \begin{itemize}
        \item query complexity $q := 2^{n-s(n)}$, for $s(n) = n/(k+1)-O(\log(\gamma^{-1}) +2^k \cdot d\log(n)\log(dn))$
        \item time complexity $t := O(q)$.
    \end{itemize}
    Here, $\gamma := \opt{{\sf PT}{\text -}{\sf Dt}[k,d]}$. 
\end{theorem}
\begin{proof}
    Immediate, from Theorem \ref{thm:agnostic-rCC} and \ref{thm:Nisan-upperbounds}.
\end{proof}

\paragraph{Remark.} By setting $m := n^{0.99}$, $d := n^{0.99}$, $k := O(1)$, and $\gamma := O(2^{n^{0.99}})$, we recover the statements of the AMQ-learners for touchstone class 1 and 2, respectively, which were presented in Section \ref{intro:ourResults}.

\section{Randomized Compression, Exact Learning, Distribution-Independent Learning}\label{sec:compression}

\subsection{Randomized Boolean Function Compression}

In this section, we show how the AMQ-learners we have obtained so far can be converted into \textit{randomized} Boolean function compression algorithms. This is done via a simple technique of \cite{carmosino2016learning}, where a PAC-learning algorithm (optionally using membership queries) which learns over the uniform distribution is used to obtain a \textit{lossy} compression algorithm, and then a linear scanning procedure is used to fix the mistakes.

\begin{definition}[Compression -- written essentially as in \cite{oliveira2016conspiracies}]
Given a circuit class $\cC$, a
compression algorithm for $\cC$ is an algorithm $\cA$ for which the following hold:
\begin{itemize}
    \item Given an input $z\in \{0, 1\}^{2^n}$, $\cA$ outputs a circuit $C$ of size $o(2^n/n)$ such that if the function $f_z$ represented by the truth table $z$ is such that $f_z \in \cC$, then $C$ computes $f_z$.
    \item $\cA$ runs in time $\poly(2^n)$. 
\end{itemize}
We say $\cC$ is compressible if there is a (polynomial time) compression algorithm for $\cC$. If the algorithm $\cA$ is probabilistic, and produces a correct circuit with probability at least $2/3$ over its random choices, then $\cC$ is probabilistically compressible.
\end{definition}

\paragraph{Remark.} One note about the above definition: The constant 2/3 as the desired success probability is not particularly important. Since the runtime needs to be $2^{O(n)}$, i.e., polynomial in the input length of $2^n$, a randomized algorithm can check for success in linear time. Hence, one could change 2/3 to be $1-2^{-n}$ if desired.

\begin{theorem}\label{thm:compression}
    There exists a sufficiently large constant $K$ such that the following complexity classes are probabilistically compressible:
    \begin{enumerate}
        \item ${\sf PT}{\text -}{\sf Ckt}[k,m]$, whenever $m\log(m)\log(mn) \cdot 2^k \le n/K(k+1)$
        \item ${\sf PT}{\text -}{\sf Dt}[k,d]$, whenever $d\log(d)\log(dn) \cdot 2^k \le n/K(k+1)$
        \item $\sym^+{\text -}{\sf Ckt}[k, t, m]$, whenever $m\log(t) \cdot 2^k \le n/K(k+1)$
        \item $\sym^+{\text -}{\sf Dt}[k, t, d]$, whenever $d\log(t) \cdot 2^k \le n/K(k+1)$
    \end{enumerate}
\end{theorem}
\begin{proof}
    Given an AMQ-learner for a touchstone class $\cC$, which runs in time $T(n,\epsilon^{-1},\delta^{-1})$, we can transform it to a randomized compression algorithm for $\cC$.
    The randomized compression algorithm is as follows. Given the truth table $T \in \{0,1\}^{2^n}$ as input, use it to simulate a membership query oracle in order to run the AMQ-learner, with $\epsilon:=1/2^{n^{0.99}}, \delta:= 2/3$, on the concept defined by $T$. Let $h$ be the output of the AMQ-learner. Now, compare the truth table indicated by $h$ to $T$, and, on whatever points they disagree, hard-wire that input/output pair into $h$. Output the amended $h$.

    Let us analyze this algorithm. 
    We upper bound the runtime of the learner and the size of the output circuit. Since we are aiming for a randomized compression, we only need to consider the case that the AMQ-learner outputs a hypothesis with the desired error, which happens with probability at least 2/3. The circuit size of the hypothesis is at most the running time of the AMQ-learner, so at most $T(n,\epsilon^{-1},\delta^{-1})$. Then, after the amendment procedure, the circuit adds  $O(\epsilon2^n)$ size. Hence, overall, the runtime and size of the output circuit of the randomized compression algorithm is $T(n,\epsilon^{-1},\delta^{-1}) + O(\epsilon2^n)$.

    Finally, we can can invoke this transformation with a correct setting of $\epsilon$ as a function of $n$, using the AMQ-learners for each of the classes 1-4. By Theorems \ref{thm:agnostic-PTF}, \ref{thm:agnostic-DT}, \ref{thm:AgnosticSYM circuits} and \ref{thm:agnostic-DT}, setting $\epsilon := 2^{-n^{0.99}}$, we obtain randomized compression algorithms for 1-4, respectively. The compressed circuit has size $O(2^{n-n^{0.99}})$ in all cases, because the AMQ-learner for each class has runtime bounded above by $2^{n/C} \in O(2^{n-n^{0.99}})$ for some constant $C>1$, depending only on $k$ and $K$, and using the conditions stated for each class in the theorem.
\end{proof}

\subsection{Exact Learning with Membership and Equivalence Queries}

In the final section, we explain that our AMQ-learners can be converted into learning algorithms in the exact learning with membership and equivalence queries model \cite{angluin1988queries}, as well as learning algorithms in the distribution independent PAC model, with membership queries.

Let us start with a description of the exact learning with membership and equivalence queries model. Let $f:\{0,1\}^n \rightarrow \{-1,1\}$ be an unknown concept. Let $\cC$ be a fixed \textit{concept class}, and assume $f \in \cC$. The learning algorithm aims to output a hypothesis $h :\{0,1\}^n \rightarrow \{-1,1\}$ such that for every $x \in \{0,1\}^n$, $h(x) = f(x)$. The learning algorithm does this by frequently updating the hypothesis through a number of timesteps, where at each timestep, the learning algorithm makes one of two types of oracle queries:
\begin{itemize}
    \item A membership query: the learning algorithm selects $q \in \{0,1\}^n$ and obtains $f(q)$.
    \item An equivalence query: the learning algorithm presents the current hypothesis $h'$, and receives either ``success'' (when $h'(x)= f(x)$ for all $x$), or a counterexample $z$ such that $h'(z) \not= f(z)$.
\end{itemize}
The running time of the algorithm is the number of timesteps until the algorithm obtains the ``success'' response, in the worst case. In the randomized version, which we now consider exclusively, the running time of the algorithm is considered $T(n)$ if it the algorithm succeeds with probability $1-\delta$ and runs in time no more than $T(n) \cdot \log(\delta^{-1})$.

\begin{theorem}
    There exists a sufficiently large constant $K$ such that the following complexity classes are learnable in the randomized exact learning with membership and equivalence queries model, in time $\poly(n) \cdot 2^{n-n^{0.99}}$.
    \begin{enumerate}
        \item ${\sf PT}{\text -}{\sf Ckt}[k,m]$, whenever $m\log(m)\log(mn) \cdot 2^k \le n/K(k+1)$
        \item ${\sf PT}{\text -}{\sf Dt}[k,d]$, whenever $d\log(d)\log(dn) \cdot 2^k \le n/K(k+1)$
        \item $\sym^+{\text -}{\sf Ckt}[k, t, m]$, whenever $m\log(t) \cdot 2^k \le n/K(k+1)$
        \item $\sym^+{\text -}{\sf Dt}[k, t, d]$, whenever $d\log(t) \cdot 2^k \le n/K(k+1)$
    \end{enumerate}
\end{theorem}
\begin{proof}
The proof flows identically to the proof of Theorem \ref{thm:compression}, except instead of doing a linear scan of the input truth table to amend mistakes, we invoke the equivalence query oracle to obtain a point $z$ that needs to be amended (by negating the existing value). It takes at most $\poly(n)$ time to amend each mistake. One can reduce the failure property at an exponential rate via testing and repeating, so the randomized algorithm incurs only a multiplicative factor of $\log(\delta^{-1})$ in runtime.
\end{proof}

\subsection{Distribution Independent PAC-Learning with Membership Queries}

It is well known since \cite{angluin1988queries} that a learning algorithm in the exact learning with membership and equivalence queries model implies a distribution independent PAC-learning algorithm, which also uses membership queries. The basic idea is that equivalence queries can be simulated, up to some small probability of an incorrect answer, with random queries. One can do this by choosing sufficiently many random queries, and realizing that if the current hypothesis is more than $\epsilon$-far from the unknown concept, then the random sample should contain a counterexample with high probability (on the other hand, if the current hypothesis is more than $\epsilon$-far from the unknown concept, then the algorithm can terminate since we are in the PAC model of learning).

Quantitatively, one can transform an exact learning with membership and equivalence queries algorithm for a concept class $\cC$ that runs in time $T(n)$ and uses $Q(n)$ queries into a distribution-independent PAC-learning with membership queries algorithm that runs in time $O(T(n)/\epsilon\cdot \ln{T(n)/\delta})$ and uses $O(Q(n)/\epsilon\cdot \ln{Q(n)/\delta})$ queries (where PAC-identification occurs with accuracy $1-\epsilon$ and confidence $1-\delta$). We refer the reader to \cite{servedio2017circuit} and Section 2.4 of \cite{angluin1988queries} for further details.

We present the following without formal proof.

\begin{theorem}
    There exists a sufficiently large constant $K$ such that the following complexity classes are learnable, to accuracy $1-\epsilon$ and confidence $1-\delta$, in the distribution-independent PAC-learning with membership queries model, in time $\poly(n)\cdot 2^{n-n^{0.99}}/(\epsilon\delta)$:
    \begin{enumerate}
        \item ${\sf PT}{\text -}{\sf Ckt}[k,m]$, whenever $m\log(m)\log(mn) \cdot 2^k \le n/K(k+1)$
        \item ${\sf PT}{\text -}{\sf Dt}[k,d]$, whenever $d\log(d)\log(dn) \cdot 2^k \le n/K(k+1)$
        \item $\sym^+{\text -}{\sf Ckt}[k, t, m]$, whenever $m\log(t) \cdot 2^k \le n/K(k+1)$
        \item $\sym^+{\text -}{\sf Dt}[k, t, d]$, whenever $d\log(t) \cdot 2^k \le n/K(k+1)$
    \end{enumerate}
\end{theorem}

\section*{Acknowledgements}I'm grateful to Marco Carmosino for insightful discussions relating to this line of work. Part of this work was completed while I was visiting the Simons institute for the theory of computing.

\bibliography{bib.bib}

\end{document}